\documentclass{article}

\usepackage[preprint]{neurips_2025}

\usepackage{mathtools}
\mathtoolsset{showonlyrefs}
\usepackage[utf8]{inputenc} 
\usepackage[T1]{fontenc}    
\usepackage{hyperref}       
\usepackage{url}            
\usepackage{booktabs}       
\usepackage{amsfonts}       
\usepackage{nicefrac}       
\usepackage{microtype}      
\usepackage{xcolor}         

\usepackage{amsmath}
\usepackage{theorem}
\usepackage{graphicx}
\usepackage{multirow}

\newcommand{\cL}{{\mathcal L}}
\newcommand{\R}{{\mathbb R}}
\newcommand{\by}{{\bf y}}
\newcommand{\bx}{{\bf x}}
\newcommand{\tr}{\operatorname{tr}}
\newcommand{\cU}{{\mathcal U}}
\newcommand{\eps}{{\epsilon}}

\newtheorem{theorem}{Theorem}
\newtheorem{proposition}[theorem]{Proposition}

\newenvironment{proof}[1][\proofname]{
\par\normalfont\trivlist\item[\hskip\labelsep\textbf{#1}.]\ignorespaces}
{\hfill $\square$ 
\endtrivlist}
\newcommand{\proofname}{Proof}

\title{Training NTK to Generalize with KARE}

\author{%
  Johannes Schwab \\
  Ecole Polytechnique Fédérale de Lausanne\\
  Quartier UNIL-Dorigny, Extranef 214\\
    CH – 1015 Lausanne, Switzerland\\
  \texttt{johannes.schwab@epfl.ch}
  \AND
  Bryan Kelly\\
  Yale School of Management\\
  165 Whitney Ave. \\
  New Haven, CT 06511\\
  \texttt{bryan.kelly@yale.edu} \\
   \And
   Semyon Malamud\\
   Ecole Polytechnique Fédérale de Lausanne \\
   Quartier UNIL-Dorigny, Extranef 213\\
    CH – 1015 Lausanne, Switzerland \\
  \texttt{semyon.malamud@epfl.ch} 
  \And Teng Andrea Xu\\
  AQR Capital Management\\
  15 Bedford St.
  London, United Kingdom.\\
  andrea.xu@aqr.com
}

\begin{document}

\maketitle

\begin{abstract}
The performance of the data-dependent neural tangent kernel (NTK; \cite{jacot2018neural}) associated with a trained deep neural network (DNN) often matches or exceeds that of the full network. This implies that DNN training via gradient descent implicitly performs kernel learning by optimizing the NTK. In this paper, we propose instead to optimize the NTK explicitly. Rather than minimizing empirical risk, we train the NTK to minimize its generalization error using the recently developed Kernel Alignment Risk Estimator (KARE; \cite{jacot2020kernel}). Our simulations and real data experiments show that NTKs trained with KARE consistently match or significantly outperform the original DNN and the DNN-induced NTK (the after-kernel). These results suggest that explicitly trained kernels can outperform traditional end-to-end DNN optimization in certain settings, challenging the conventional dominance of DNNs. We argue that explicit training of NTK is a form of over-parametrized feature learning.
\end{abstract}

\section{Introduction}\label{sec:intro}

We consider the classical framework of supervised learning based on empirical risk minimization (ERM). Given training data consisting of input-output pairs \( \ (x_i, y_i) \in \R^d\times \R, i\ =\ 1, \ldots, n \), a parametric model \(f(x ; \theta): \R^d \to \R \), and a loss function \(\ell: \R\times \R \to \R_+\), ERM seeks
\begin{equation}\label{loss}
\theta_* = \arg\min_\theta \cL(\theta),\  \cL(\theta)\ =\ \frac1n\sum_{i=1}^n \ell(y_i, f(x_i; \theta))\,.
\end{equation}

Under standard assumptions--such as low model complexity and i.i.d.\ data--ERM yields good generalization: the empirical risk approximates the expected risk,
\begin{equation}
\mathbb{E}_{\bx,\by}[\ell(\by, f(\bx; \theta))],
\end{equation}
ensuring predictive performance on unseen data.

However, deep neural networks (DNNs) violate these assumptions. Modern architectures are overparameterized and often interpolate the training data, achieving zero training loss. In such settings, minimizing empirical risk does not guarantee low expected risk. Ideally, we would minimize expected risk directly, but this is typically infeasible.

\textit{In this paper}, we contribute to the literature by demonstrating that the after-NTK can be naturally interpreted as a ``\textit{gradient booster}" built on top of a DNN. Motivated by this theoretical insight, we propose a novel machine learning approach that directly leverages this gradient-boosting perspective to minimize the generalization error of deep neural networks.

To minimize the generalization error, we leverage two insights from the DNN literature. The first is the \emph{Neural Tangent Kernel (NTK)}~\cite{jacot2018neural}, which is a powerful analytical tool for characterizing the behaviors of wide DNNs. The second is the \emph{Kernel Alignment Risk Estimator (KARE)} introduced by \cite{jacot2020kernel}, which consistently estimates the {\it expected risk} of a kernel ridge regression, particularly for overparameterized kernel machines such as NTK.  

The key insight of our approach is two-fold: 1) we can avoid training the DNN outright and instead explicitly train its NTK, and 2) by using KARE (rather than ERM) as the training objective, we can directly train the NTK to generalize. We refer to our generalization-driven kernel learning approach as {\it NTK-KARE}.  

Empirically, we compare NTK-KARE to the corresponding trained DNN and its after-training NTK.  
Consistent with the prior literature, we find that the after-training NTK without KARE performs similarly to the trained DNN. However, the generalization performance of NTK-KARE consistently exceeds that of the DNN and after-training NTK across simulations and real datasets.

This paper is structured as follows. Section~\ref{sec:literature} reviews recent theoretical and empirical findings related to the NTK and after-NTK. In Section~\ref{sec:prelim}, we formally introduce relevant mathematical concepts, present our theoretical contribution, and propose a novel algorithm to explicitly train the NTK. Section~\ref{sec:experiments} presents empirical results, while Section~\ref{sec:discussion} discusses the limitations of our approach and outlines directions for future research.

\section{Related Work}
\label{sec:literature}

The emergence of the Neural Tangent Kernel (NTK; \cite{jacot2018neural}) and its intricate connection to gradient descent dynamics (\cite{chizat1812note}; \cite{daniely2016toward, daniely2017sgd}; \cite{arora2019fine}) has paved new pathways for understanding the behavior of DNNs. To date, however, NTK has served as a theoretical tool to describe DNN training dynamics and generalization, often summarizing the properties of a DNN at initialization or once the DNN has been trained (the ``after-training" NTK).  

In contrast, we leverage NTK as a machine learning model to be trained directly. The consequences of this are two-fold. First, with parameterization and structure that mirrors DNNs, NTK-KARE further migrates the landscape of kernel learning from the classical setting of relatively low-dimensional models--e.g., the Gaussian kernel
$K(x,\tilde{x};\theta)=\exp(-\|x-\tilde{x}\|^2/\theta^2)$ with parameter $\theta\in\mathbb{R}$ and the anisotropic kernel $K(x,\tilde{x};\theta)=\exp(-(x-\tilde{x})^\top\theta(x-\tilde{x})),\ \theta\in\mathbb{R}^{d\times d}$ (as in \cite{radhakrishnan2024mechanism})--to the setting of heavily parameterized kernels with larger expressive power and scalability (\cite{wilson2016deep}). 
The second consequence is that, thanks to its natural conjugacy with KARE, we can directly optimize NTK's generalizability.  
Unlike DNNs, kernel regressions allow approximation of generalization error via data-dependent surrogates as in KARE\footnote{KARE is closely related to the classical Generalized Cross-Validation (GCV) estimator of \cite{golub1979generalized, craven1978smoothing}. The consistency of the GCV estimator in high-dimensional settings has been studied in numerous papers; see, e.g., \cite{hastie2019surprises, xu2021consistent, patil2021uniform}.} (\cite{jacot2020kernel}, see also \cite{wei2022more} and \cite{misiakiewicz2024non}).  In fact, our KARE-based training approach can be coupled with any kernel machine, including that of (\cite{wilson2016deep}).

A growing body of empirical work documents that the after-training NTK is highly predictive {\it without being directly trained}, often outperforming the trained DNN from which the NTK is derived (\cite{fort2020deep, geiger2020disentangling, atanasov2021neural, vyas2022limitations, lauditi2025adaptive}). This hints at fascinating potential for the directly-trained NTK.

An emerging literature emphasizes that feature learning is crucial for the success of modern DNNs. A particularly notable contribution by \cite{radhakrishnan2024mechanism} demonstrates that the mechanism behind feature learning in DNNs can be understood as a form of kernel learning. Traditionally, however, kernel learning has focused on kernels with low {\it feature complexity}--defined as the ratio of the number of learnable kernel parameters to the number of kernel features.\footnote{The kernel function can be viewed as an infinite-dimensional feature machine (see, e.g., \cite{simon2023more}), thus the effective complexity of a classical kernel machine such as that in \cite{radhakrishnan2024mechanism} is zero.}  In contrast, the Neural Tangent Kernel (NTK) possesses a {\it feature complexity of one} because $\dim(\theta)=\dim(\nabla_\theta f)$.\footnote{\cite{jacot2018neural} establish that in the infinite-width limit with standard initialization, NTK remains nearly constant throughout training. Consequently, the DNN converges to a kernel regression predictor using a fixed kernel. As shown by \cite{yang2021tensor}, networks trained in the constant-NTK regime lack the ability to learn features. For more practical finite-width networks, the NTK evolves significantly throughout training (\cite{vyas2022limitations}), and this evolution is critical to achieving strong generalization.} 
Over-parametrized kernel machines, derived from DNNs but directly optimizable, open new avenues for understanding the role of feature learning, over-parametrization, and generalization in machine learning (see, e.g., \cite{belkin2018reconciling, belkin2018understand, hastie2019surprises}).

\section{Main results}\label{sec:prelim}
\subsection{Preliminaries}
\noindent \textbf{Kernel Ridge Regression (KRR).} Throughout this paper, we adopt standard vector notation, defining the data matrix $X\ =\ [x_1, x_2, \ldots, x_n]^\top \in \R^{n\times d},$ as the collection of in-sample stacked observations, and the vector $y\ =\ [y_1, y_2, \ldots, y_n]^\top \in \R^n$ as the corresponding response variables. Thus, given a positive-definite kernel function $K,$ Kernel Ridge Regression (KRR) explicitly yields the prediction function as
\begin{equation}
    \hat y(x) = \frac{1}{n}K(x, X) \Big( \frac{1}{n} K(X,X) + zI_n \Big)^{-1} y, 
\end{equation}
where $z\in \R$ is the ridge penalty parameter, $K(x, X) = (K(x,x_i))_{i = 1}^n \in \R^n$  is the vector of kernel evaluations between a point $x$ and training points, and $K(X,X) = (K(x_i, x_j))_{i, j=1}^n \in \R^{n\times n}$ is the kernel matrix evaluated on the training data..

\noindent \textbf{NTK and data-dependent NTK.} We refer to $f(x;\theta)$ as ``the DNN" parametrized by $\theta\in \R^P$ whose corresponding NTK~(\cite{jacot2018neural} is defined as 
\begin{equation}\label{NTK-def}
K(x, \tilde{x}; \theta) = \nabla_\theta f(x; \theta)^\top \nabla_\theta f(\tilde{x}; \theta).
\end{equation}
It has been shown that, at initialization, the NTK remains nearly constant for sufficiently wide neural networks (see, e.g.,~\cite{jacot2018neural, lee2019wide, zou2019improved, du2019gradient}). In practice, however, for finite-width neural networks trained over an extended period, the empirical kernel obtained after training, often referred to as the data-dependent NTK~\cite{fort2020deep}, typically exhibits substantial changes relative to its initial state. Training the DNN amounts to optimizing the parameters $\theta$ using gradient descent: \begin{equation}\label{gd1}
\theta_{t+1}\ =\ \theta_t - \eta\,\nabla_\theta\mathcal{L}(\theta_t)\,,
\end{equation}
where $t\in \mathbb N_+$ denotes the iteration step and $\eta > 0$ is the learning rate. We denote by $\theta_T$ the optimized DNN parameters after $T$ steps and define the data-dependent NTK as
\begin{equation}\label{eq:after-ntk}
K(x, \tilde{x}; \theta_T) = \nabla_{\theta_T} f(x; \theta_T)^\top \nabla_\theta f(\tilde{x}; \theta_T).
\end{equation}

\subsection{{Data-Dependent NTK: Gradient Boosting}}
\label{sec:gradient-boosting}
We now show how the after-training NTK emerges as a form of ``gradient boosting" and then discuss how, perhaps surprisingly, it is the main event and can be used to replace the original DNN entirely. 

It is known that the NTK evolves as the parameters update during gradient descent, eventually stabilizing in the latter stages of DNN training~\cite{fort2020deep}. At that point, the model admits a natural decomposition:

\begin{proposition}\label{prop:gron} Let $\hat y = f(X; \theta)$ and $\ell_{\hat y} \coloneq \nabla_{\hat y} \ell$. Suppose that $f$ is differentiable, $\ell_{\hat y}$ is continuous, and that $\ell$ is such that, for any $R>0,$ the set $\{v\in \R:\ell(y,v)<R\}$ is bounded and $\ell\ge -A$ for some $A>0$. Suppose also that NTK stabilizes after $T$ epochs: $\|K(x;X;\theta_t)-K(x;X;\theta_T)\|\le \eps$ for all $t\ge T.$ Then, 
\begin{equation}\label{main-dec}
f(x;\theta_t)\ =\ \underbrace{f(x;\theta_{T})}_{trained\ DNN}\ +\underbrace{K(x,X;\theta_{T})\,\cU_t}_{trained\ kernel\ machine}\ +\ O(\eps)\,
\end{equation}
for some vector $\cU_t$ that depends on the training data.
\end{proposition}
\begin{proof}
We begin by considering the first-order Taylor expansion of the DNN $f(x; \theta_{t+1})$ after $t+1$ gradient descent updates, as defined in equation~\eqref{gd1}: 
\begin{equation}\label{ode0}
\begin{aligned}
f(x;\theta_{t+1})\ &=\ f(x;  \theta_t\ -\ \eta\,\nabla_\theta\cL(\theta_t))\\
&\hspace{-0.4em}\underbrace{\approx}_{\mathclap{\text{1st order Taylor Approx.}}}\,f(x;  \theta_t)\ -\ \eta \nabla_\theta f(x;\theta_t)^\top \nabla_\theta\cL(\theta_t)\,.
\end{aligned}
\end{equation}
Now, 
\begin{equation}
\begin{aligned}
&\nabla_\theta\cL(\theta)\ =\ \nabla_\theta \left(\frac1n \sum_{i=1}^n \ell(y_i,f(x_i;\theta))\right)\\
&=\ \frac1n \sum_{i=1}^n \nabla_\theta\ell(y_i,f(x_i;\theta))\\
&\ =\ \frac1n\sum_{i=1}^n \ell_{\hat y}(y_i,f(x_i;\theta))\, \nabla_\theta f(x_i;\theta)
\end{aligned}
\end{equation}
Substituting, we obtain {\it gradient descent in the prediction space:}
\begin{equation}
\begin{aligned}
&\nabla_\theta f(x;\theta_t)' \nabla_\theta\cL(\theta_t)\\
&=\ \nabla_\theta f(x;\theta_t)'  \left( {\frac{1}{n}} \sum_{i=1}^n \ell_{\hat y}(y_i,f(x_i;\theta_t))\, \nabla_\theta f(x_i;\theta_t)\right)\\
&=\ {\frac{1}{n}} \sum_{i=1}^n \ell_{\hat y}(y_i,f(x_i;\theta_t))\, \nabla_\theta f(x;\theta_t)^\top \nabla_\theta f(x_i;\theta_t)\,.
\end{aligned}
\end{equation}
Using the NTK definition~\eqref{NTK-def}, replacing $\eta$ with $\eta dt,$ and taking the limit as $dt\to 0$, we can rewrite \eqref{ode0} as
\begin{equation}\label{grad-flow} \frac{d}{dt}f(x;\theta_t) = -\frac{\eta}{n}\,K(x,X;\theta_t) \,\ell_{\hat y}(y;f(X;\theta_t))\,, 
\end{equation} 
Observe that 
\begin{equation}
\frac{d}{dt}\sum_{i=1}^n \ell(y_i,f(X_i;\theta_t))\ =\ -\ell_{\hat y}(y;f(X;\theta_t))^\top \eta K(x,X;\theta_t)\,\ell_{\hat y}(y;f(X;\theta_t))\ \le\ 0\,
\end{equation}
because $K$ is positive semi-definite. From the assumptions made about $\ell,$ we immediately get that $f(X;\theta_t)$ stays uniformly bounded. Let $\check f$ satisfy 
\begin{equation}
\begin{aligned}
&\frac{d}{dt}\check f(x;\theta_t)\ =\ -\eta K(x,X;\theta_T)\,\frac1n\ell_{\hat y}(y;f(X;\theta_t))\ ,\ \check f(x;\theta_T)=f(x;\theta_T)\,. 
\end{aligned}
\end{equation}
Then, 
\begin{equation}
\begin{aligned}
&\|\check f(x;\theta_t)-f(x;\theta_t)\|\ =\ \|\int_t^T\frac{\eta}{n} (K(x,X;\theta_\tau)-K(x,X;\theta_T))\ell_{\hat y}(y;f(X;\theta_\tau))d\tau\|\\
&\ \le\ \left|\int_t^T\frac{\eta}{n} \|(K(x,X;\theta_\tau)-K(x,X;\theta_T))\ell_{\hat y}(y;f(X;\theta_\tau))\|d\tau\right|\\
&\ \le\ \eps\,\frac{\eta}{n} (t-T)\sup_{\tau\in [T,t]}\|\ell_{\hat y}(y;f(X;\theta_\tau))\|,
\end{aligned}
\end{equation}
where the latter supremum is finite because $\ell_{\hat y}$ is continuous and $f(X;\theta_t)$ stays uniformly bounded\,. The claim now follows because 
\begin{equation}
\check f(x;\theta_t)\ =\ f(x;\theta_T)\, {+}\, K(x,X;\theta_T)\frac{\eta}{n}  \underbrace{\int_t^T\ell_{\hat y}(y;f(X;\theta_\tau))d\tau}_{\mathcal{U}_t}\,.
\end{equation}

\end{proof}

The link between NTK and DNN is particularly clear for the MSE loss $\ell(y_i,\hat y_i)=(y_i -\hat y_i)^2.$ In this case, \eqref{main-dec} takes the form
\begin{equation}
f(x;\theta_t)\ \approx\ f(x;\theta_T)\ +\ \frac1n K(x,X;\theta_{T}) (zI+\frac1n K(X,X;\theta_{T}))^{-1}(y-f(X;\theta_T))
\end{equation}
for some ridge parameter $z.$  In other words, residual DNN training (for $t>T$) experiences a form of ``gradient boosting" in which  DNN residuals $y-f(X;\theta_T)$ are fit via kernel ridge regression (the kernel ridge predictor uses $\mathcal{U}(z) = (n^{-1} K(X, X;\theta_T) + z I)^{-1} y$ rather than the implicit $\mathcal{U}_t$ of \eqref{main-dec}). Note that the after-training NTK in \eqref{main-dec} is not directly optimized; it is just evaluated at the trained DNN parameters.

A striking discovery of \cite{fort2020deep, geiger2020disentangling, atanasov2021neural, vyas2022limitations, lauditi2025adaptive} is that replacing $f(x;\theta_T)$ with $0$ and changing $z$ with a judiciously chosen $\tilde z$ gives approximately the same result, 
\begin{equation}\label{main-equiv}
\begin{aligned}
f(x;\theta_T)\ &+\ \frac1n K(x,X;\theta_{T}) (zI+\frac1n 
 K(X,X;\theta_{T}))^{-1}(y-f(X;\theta_T))\\
&\approx\ \frac1n  K(x,X;\theta_{T}) (\tilde zI+\frac1n 
 K(X,X;\theta_{T}))^{-1}y\,. 
\end{aligned}
\end{equation}
In particular, the prediction performance of the after-training NTK $K(\cdot,\cdot;\theta_T)$ often matches or surpasses the DNN predictor $f(x;\theta_T)$. Evidently, the kernel component is not just a booster; it is the main event.  

\subsection{Training NTK \textit{Explicitly}}

As discussed above, the after-training NTK arises \emph{implicitly} as a consequence of training the DNN, yet it matches the performance of the trained network (as expressed by equation \eqref{main-equiv}). This motivates us to devise an approach to train NTK {\it explicitly.} To deal with the divergence between the empirical risk and generalization risk discussed in Section \ref{sec:intro}, we propose a novel algorithm for training kernels that minimizes an approximation of the expected risk recently introduced in~\cite{jacot2020kernel}. 

Given a positive definite kernel $K(\cdot,\cdot)$, a kernel matrix \( K(X, X)\in \R^{n\times n} \) with $X\in \R^{n\times d}$, labels \( y\in \R^n \), and a regularization parameter \( z\in \R_+\), \cite{jacot2020kernel} define an expected risk estimator called the Kernel Alignment Risk Estimator (KARE) as follows
\begin{equation}\label{KARE-def}
\text{KARE}(y, K(X,X), z) = \frac{\frac{1}{n} y^\top \left( \frac{1}{n} K(X,X) + z I_n \right)^{-2} y}
{\left( \frac{1}{n} \tr \left[ \left( \frac{1}{n} K(X,X) + z I_n \right)^{-1} \right] \right)^2}\,.
\end{equation}
\cite{jacot2020kernel} and \cite{misiakiewicz2024non} show that KARE closely approximates the expected risk of the kernel machine
\begin{equation}
\mathbb{E}_{\bx,\by}[(\by - K(\bx,X)(z I_n + K(X,X))^{-1}y)^2] \approx \text{KARE}(y, K(X,X), z).
\end{equation}
Thus, one can optimize a kernel $K(x,\tilde x;\theta)$ explicitly by training a neural network  $f(x; \theta^{\text{KARE}})$ to minimize KARE via gradient descent
\begin{equation}\label{KARE-gd-def}
\theta_{t+1}^{\rm KARE} = \theta_t^{\rm KARE} - \eta_{\rm KARE} \nabla_\theta \text{KARE}(y, K(X,X; \theta_t^{\rm KARE}), z_{\rm KARE}).
\end{equation}
This optimization yields the following estimator, which we term NTK-KARE
\begin{equation}\label{NTK-KARE}
\text{NTK-KARE}(x)\ =\ \frac{1}{n} K(x, X; \theta_T^{\text{\rm KARE}}) \left(n^{-1}K(X,X; \theta_T^{\text{\rm KARE}}) + \lambda I_n\right)^{-1} y, \\ 
\end{equation}
where $K(X,X; \theta_T^{\text{KARE}})\ =\ \nabla_\theta f(x; \theta_T^{\text{KARE}})^\top \nabla_\theta f(\tilde{x}; \theta_T^{\text{KARE}}).$

Note that NTK-KARE is \emph{unaware} of the parametric model $f(x;\theta)$; it only \emph{sees} the kernel $K$ and directly optimizes it. Thus, it implements a form of \emph{pure kernel learning}.

In the remainder of this paper, we demonstrate that explicitly training the NTK—thereby obtaining the NTK-KARE estimator defined in \eqref{NTK-KARE}—results in superior performance compared to the after-training NTK \eqref{eq:after-ntk} and the standard DNN across various settings. For clarity, we refer to the "DNN" as a standard neural network with parameters $\theta_T$ optimized via SGD by minimizing either the MSE loss or cross-entropy loss, depending on the task. We define the after-NTK estimator explicitly as follows:
\begin{equation}\label{after-NTK-est}
\text{after-NTK}(x)\ =\ \frac{1}{n} K(x, X; \theta_T) \left(n^{-1}K(X,X; \theta_T) + \lambda I_n\right)^{-1} y. \\ 
\end{equation}

See Tables ~\ref{tab:hyperparam-exp} and ~\ref{tab:hyperparam-exp-uci-openml} in the Appendix for the specific hyperparameters $\eta^{\text{\rm KARE}}$ and $z_{\text{\rm KARE}}$ used to run equation~\eqref{KARE-gd-def}.  
For both kernel predictors, the regularization parameter $\lambda$ is set proportionally to the normalized trace of the kernel matrix
\begin{equation}
\begin{aligned}
\lambda = \tilde \lambda \frac{1}{n^2}\tr(K(X,X;\theta_T))
\end{aligned}
\end{equation}
where $\tilde \lambda$ is a hyperparameter. Note that $n^{-1}\tr(K(X,X;\theta_T))$ converges to the trace of the kernel operator. Thus, $n^{-2}\tr(K(X,X;\theta_T))$ is the ``average eigenvalue" of this operator.

Additionally, one could consider a residual kernel machine of the form 
\begin{equation}
\begin{aligned}
    \hat f_{res}(x) = f(x;\theta_T) + K(x, X; \theta_T) (n^{-1}K(X,X;\theta_T) + \lambda I_n)^{-1}(y - f(X;\theta_T)).
\end{aligned}
\end{equation}
However, our experiments show that the plain kernel machine consistently outperforms this version--suggesting the DNN output $f(x;\theta_T) $ is \textit{redundant} once the kernel is known.


\section{Experiments}
\label{sec:experiments}
In this section, we compare NTK-KARE performance in both simulated and real datasets. First, we validate our theoretical results through a controlled experiment, demonstrating that NTK-KARE outperforms both the after-NTK and the standard DNN in learning a low-rank function. Second, moving to real datasets, and following~\cite{jacot2020kernel} analysis, we conduct similar evaluations with the MNIST and Higgs datasets. Finally, we assess NTK-KARE using the UCI dataset challenge~\cite{Delgado2014}, where we show that NTK-KARE not only outperforms the after-NTK and the DNN, but also matches the performance of Recursive Feature Machines (RFMs), the current state-of-the-art benchmark introduced by~\cite{radhakrishnan2024mechanism}.

We make all experiments reproducible, with code available at: \url{https://github.com/DjoFE2021/ntk-kare}.

A detailed description of the set of hyperparameters, the methodology for model selection, and the computational resources utilized in our experiments can be found in Appendix~\ref{app:num}.



\subsection{Simulated Data}
\label{sec:method}
\begin{figure}
\centering
\includegraphics[width=1\linewidth]{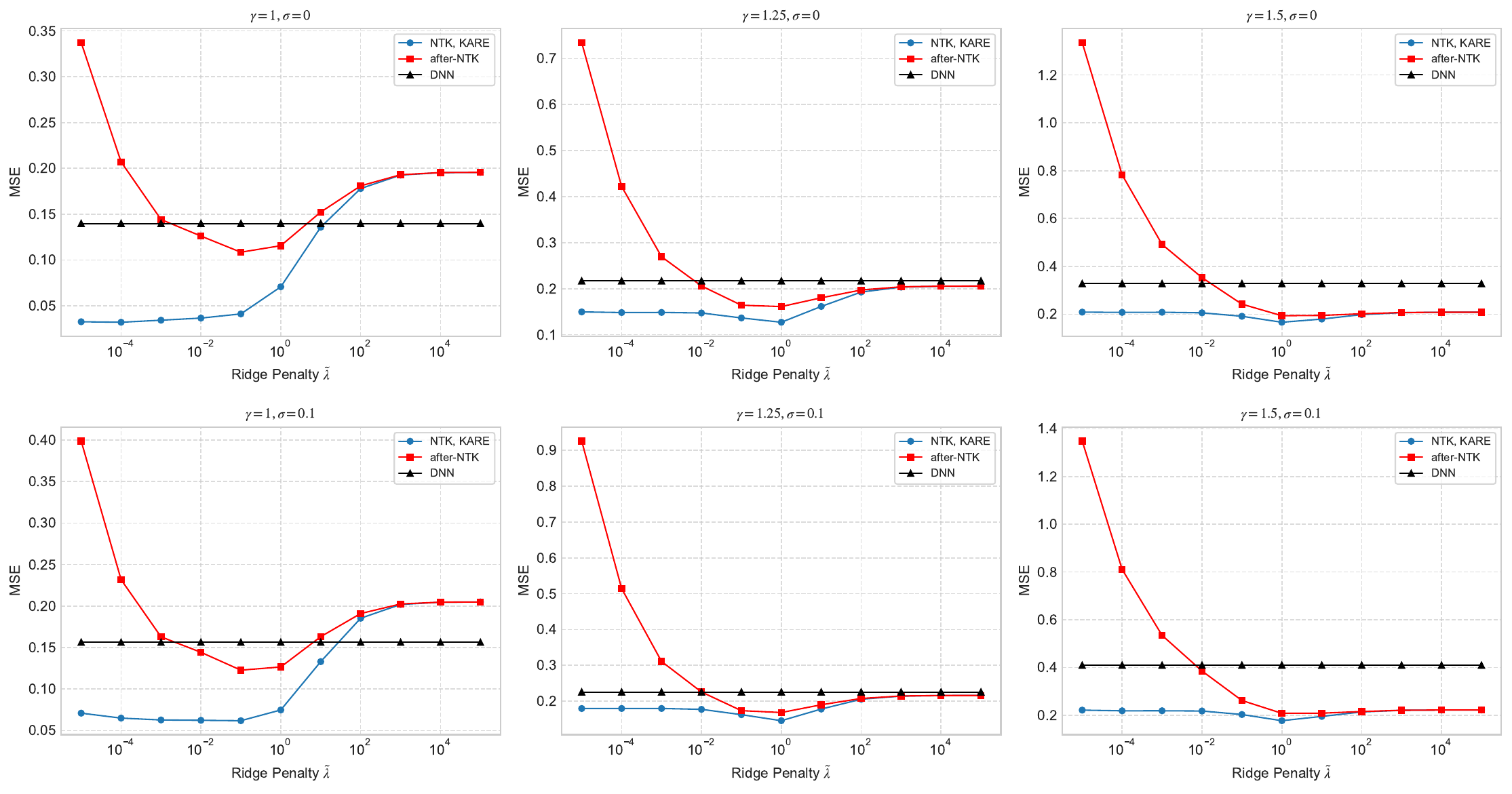}
\caption{Out-of-sample MSE on the synthetic dataset generated according to~\eqref{eq:DGP-sim} for NTK-KARE, after-NTK, and the DNN. We consider various parameterizations of the DGP by choosing $\gamma \in {1.0, 1.25, 1.5}$ (columns) and $\sigma \in {0.0, 0.1}$ (rows). NTK-KARE was fit using $z_{\rm KARE} = 0.1$. All other hyperparameters are provided in Table~\ref{tab:hyperparam-exp}.}
\label{fig:sim}
\end{figure}
\cite{ghorbani2020neural} shows that kernel methods typically fail to learn low-rank functions in small samples efficiently, and \cite{pmlr-v97-rahaman19a} shows that the difficulty of approximating a function depends on how quickly it oscillates. Motivated by these results, we study a synthetic regression problem where the response variable \( y_i \) is generated according to  
\begin{equation}\label{eq:DGP-sim}  
y_i = f_*(x_i)\ +\ \epsilon_i, \quad i=1,\dots,n,\ x_i\sim \mathcal{N}(0,I_{d\times d}),  
\end{equation}  
where the true function $f_*:\R^d\ \to \R$ is defined as \( f_*(x)= \cos(\gamma\, x^\top w) / (1 + e^{\gamma\, x^\top w}) \) with the parameter \( \gamma \) controlling the frequency of oscillation and the weight vector \( w \in \mathbb{R}^d \) sampled independently as \( w\sim \mathcal{N}(0,I_{d\times d}) \). The noise terms \( \epsilon_i \) are i.i.d. draws from a Gaussian distribution \( \mathcal{N}(0, \sigma^2) \). We consider \( \gamma \in \{1,1.5\} \) and $\sigma \in \{0,0.1\}$. We set the dimension to \( d=10 \), sufficiently large to trigger the curse of dimensionality, and choose a relatively small sample size \( n=1'000 \). Both the neural network and kernel methods employed are based on a one-layer network with 64 neurons. We repeat the experiment $k=10$ times for each setting and report the average out-of-sample MSE across these runs. 

Figure \ref{fig:sim} shows that the best generalizer is NTK-KARE, followed by the after-training kernel NTK, and both outperform the original trained network in every experiment we run. 

\subsection{MNIST and Higgs Datasets}
\begin{figure}
\centering
\includegraphics[width=1\linewidth]{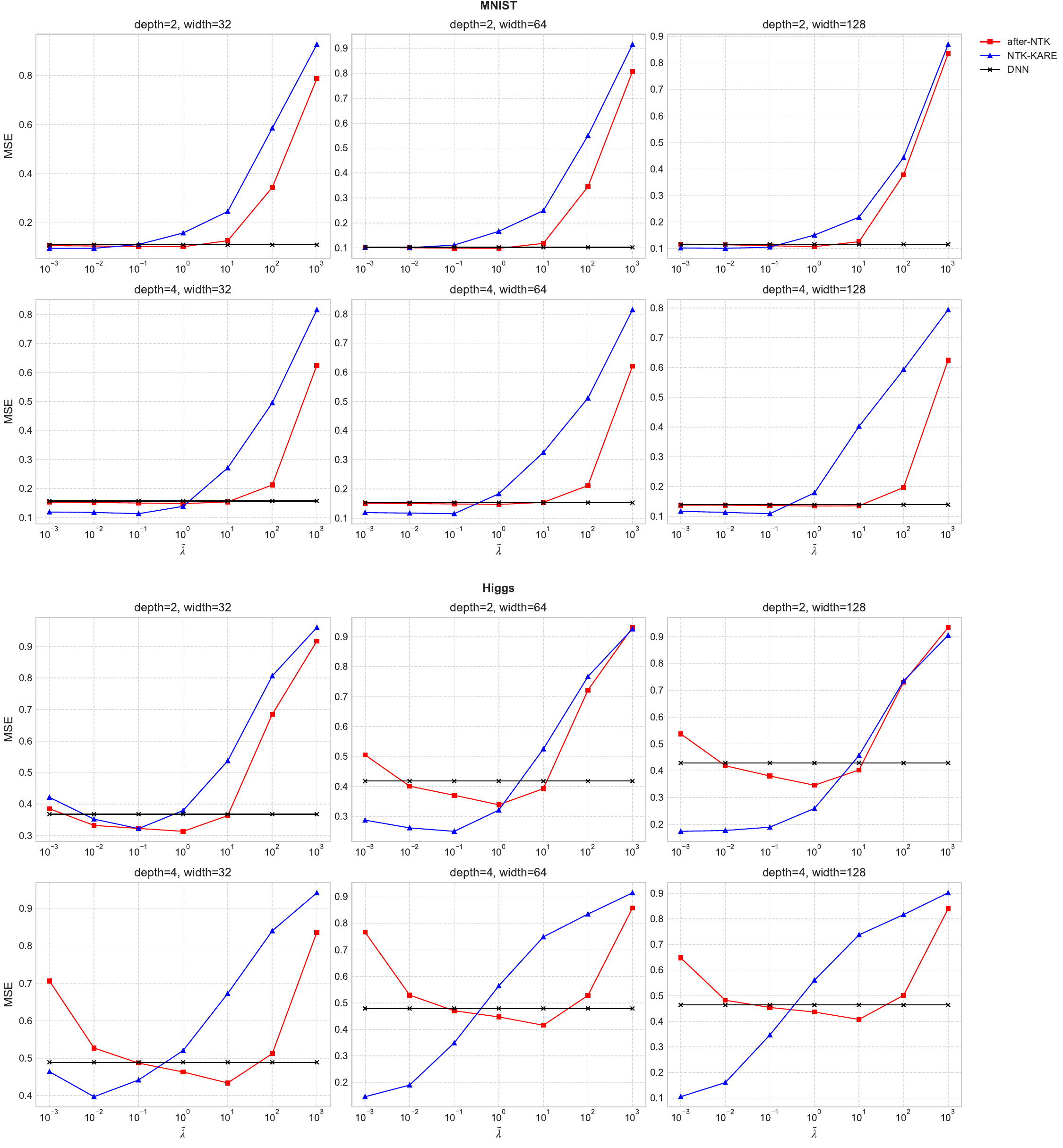}
\caption{Out-of-sample mean-squared-error on the Mnist and Higgs dataset for NTK-KARE, after-NTK, and the DNN model using $n=1'000$ observations to train the models. We consider various DNN architecture by varying network $\text{depth} \in \{2, 4 \}$ (rows) and $\text{width} \in \{32, 64, 128 \}$ (columns). We used $z_{\rm KARE} = 0.1$ to fit NTK-KARE. All the other hyperparameters can be found in Table \ref{tab:hyperparam-exp}}
\label{fig:higgs-mnist}
\end{figure}
Motivated by the analysis in \cite{jacot2020kernel}, we compare the performance of our kernel methods on the MNIST and Higgs datasets. See Section \ref{det-sec} in the Appendix for details. Figure \ref{fig:higgs-mnist} reports our results. For each architecture, we observe that NTK-KARE consistently outperforms both the DNN and after-NTK. We note that the choice of an optimal ridge penalty is essential for kernel ridge regression, consistent with the findings of \cite{simon2023more} and \cite{kelly2024virtue} who show that kernels associated with large DNNs perform best with an optimally chosen ridge penalty. As above, we run gradient descent $k=10$ times and report the MSE averaged across these $k$ runs. On the Higgs dataset, NTK-KARE attains lower out-of-sample MSE than the other methods across all architectures at the optimal ridge penalty. On MNIST, which is known to be a relatively easy benchmark and thus offers limited room for performance improvement, NTK-KARE still achieves noticeable reductions in MSE relative to both the DNN and after-NTK estimators, particularly for deeper network architectures.

\subsection{UCI Classification Datasets}
\begin{table}
  \centering
  \caption{Test results on UCI classification datasets.}
  \label{tab:res-uci}
\begin{tabular}{lccccc}
\toprule
Model & Avg. Accuracy (\%) & P90 (\%) & P95 (\%) & PMA (\%) & Friedman Rank \\
\midrule
\bf{NTK-KARE} & \bf{85.45} & 92.56 & 85.12 & 97.17 $\pm$ 5.03 & 18.34\\
\bf{after-NTK} & 85.13  & \bf{93.39} & 80.16 & 96.78 $\pm$ 4.44 & 19.52 \\
\midrule
RFM & 85.37 & 92.56 & \bf{85.96} & \bf{97.36 $\pm$ 4.04} & \bf{17.79}\\
Laplace Ridge & 83.76 & 90.08 & 74.38 & 95.95 $\pm$ 5.41 & 28.48 \\
NTK Ridge & 82.70 & 85.95 & 68.60 & 94.84 $\pm$ 8.17 & 33.55\\
Random Forest & 81.96 & 83.47 & 68.60 & 93.48 $\pm$ 12.10 & 33.52 \\
Gaussian SVM & 81.81 & 82.35 & 69.75 & 93.21 $\pm$ 11.37 & 37.50\\
DNN & 79.37 & 73.55 & 53.72 & 91.14 $\pm$ 12.81 & 44.13\\
\bottomrule
\end{tabular}
\end{table}

Finally, we show that both NTK-KARE and after-NTK achieve state-of-the-art results on the UCI dataset challenge introduced by~\cite{Delgado2014}. This challenge includes a broad variety of competing models, such as kernel methods, deep neural networks, and tree-based methods. Following the evaluation protocol of~\cite{radhakrishnan2024mechanism, arora2019harnessing}, we report several standard performance metrics in Table~\ref{tab:res-uci}, including the average accuracy across all datasets; the percentages of datasets (denoted P90 and P95) on which the classifier attains accuracy within 90\% and 95\%, respectively, of the best-performing model; the percentage of the maximum accuracy (PMA), calculated as the classifier's accuracy relative to the best model, averaged over all datasets; and the Friedman rank, measuring the average rank of the classifier across datasets. 

Table~\ref{tab:res-uci} summarizes the performance of NTK-KARE and after-NTK across 121 datasets, alongside results from standard benchmark methods reported in~\cite{radhakrishnan2024mechanism}. Consistent with our earlier findings, NTK-KARE exhibits stronger overall performance compared to after-NTK, benefiting from the \textit{gradient boosting} effect described in Section~\ref{sec:gradient-boosting}. 

NTK-KARE achieves the highest average accuracy among all considered methods and performs comparably to RFMs, a recently proposed state-of-the-art method known for its exceptional performance and powerful feature-learning capabilities, across all other metrics.

It is noteworthy that NTK-KARE outperforms after-NTK by approximately 5\% in terms of the P95 metric, while the improvements relative to NTK and the standard DNN are even more substantial, approximately 17\% and 30\%, respectively. These results clearly highlight the effectiveness of our novel \textit{kernel learning} approach: by explicitly optimizing the NTK, our proposed estimator achieves significant and non-trivial gains in performance.

All hyperparameters and model selection details are outlined in Appendix~\ref{app:num}.

\section{Discussion} 
\label{sec:discussion}
We have theoretically demonstrated that the after-training NTK can be interpreted as a form of ``gradient boosting," in which the residuals of a DNN are modeled through kernel ridge regression. Moreover, recent studies~\cite{fort2020deep, geiger2020disentangling, atanasov2021neural, vyas2022limitations, lauditi2025adaptive} have shown that by judiciously selecting the ridge penalty parameter $z$, it is possible to eliminate the dependence on the DNN altogether.

Motivated by these theoretical findings, we propose a novel algorithm, which we call NTK-KARE, to explicitly train the NTK by minimizing the Kernel Alignment Risk Estimator~\cite{jacot2018neural} to minimize generalization error. Consistent with the findings discussed above, NTK-KARE outperforms both the after-NTK and standard DNN across various datasets, achieving performance on par with that of Recursive Feature Machines on the UCI challenge. 

This perspective challenges the conventional role of DNNs by positioning NTKs as trainable, high-capacity models capable of effective feature learning.

Like other kernel methods, NTK-KARE lacks the ability to efficiently scale model size with the training dataset size, a limitation not encountered by standard DNNs. As future work, we plan to thoroughly explore existing solutions from the modern large-scale kernel learning literature, particularly those employing advanced matrix preconditioning techniques. Notable examples include FALKON~\citep{rudi2017falkon, meanti2020kernel} and EigenPro~\citep{ma2017diving, ma2019kernel, abedsoltan2023toward}.

We believe that further investigation into this novel learning algorithm can provide a robust framework for minimizing the generalization error of modern kernel methods. An important direction for future research is establishing rigorous theoretical guarantees for KARE-based training methods, which would enhance our understanding of their convergence properties and generalization performance.



\clearpage
\begin{ack}
Johannes Schwab is at Swiss Finance Institute, EPFL. Bryan Kelly is at AQR Capital Management, Yale School of Management, and NBER. Semyon Malamud is at Swiss Finance Institute, EPFL, and CEPR, and is a consultant to AQR. Andrea Xu is at AQR Capital Management. Semyon Malamud gratefully acknowledges the financial support of the Swiss Finance Institute and the Swiss National Science Foundation, Grant 100018\_192692. 

\end{ack}
\bibliographystyle{aer}
\bibliography{OLS}

\appendix 

\newpage 

\section*{Appendix}

\section{Details for Numerical Experiments}\label{app:num}

\subsection{Hyperparameters}
\begin{table}[h!]
\caption{Hyperparameters used in the simulation, Higgs, and MNIST experiments.}
\centering
\begin{tabular}{lccc}
\toprule
Hyperparameter & Simulations & Higgs & MNIST\\
\midrule
$n$ & 1'000 & 1'000 & 1'000\\
$n_{\text{test}}$ & 1'000 & 1'000 & 1'000\\
$d$ & 10 & 31 & 576 ($24\times 24$)\\
$\tilde{\lambda}$ & $10^i$, $i=-5,\dots,5$ & $10^i$, $i=-3,\dots,3$ & $10^i$, $i=-3,\dots,3$\\
$k$ & 10 & 10 & 10\\
$\text{optimizer}$ & full-batch GD & full-batch GD & full-batch GD\\
\midrule
$z_{\text{\rm KARE}}$ & 0.1 & 0.1 & 0.1\\
$epochs_{\text{\rm KARE}}$ & 100 & 300 & 300 \\
$\eta_{\text{\rm KARE}}$ & 100 & 1 if depth = 4 else 10 & 1 if depth = 4 else 10 \\
$\text{activation function}$ & GELU & GELU & GELU\\
\midrule
$\text{epochs}_{\text{MSE}}$ & 10'000 & 10'000 & 10'000\\
$\eta_{\text{MSE}}$ & 0.1 & 0.1 & 0.1\\
\midrule
$\text{depth}$ & 1 & \{2,4\} & \{2,4\}\\
$\text{width}$ & 64 & \{32,64,128\} & \{32,64,128\}\\
\midrule
$\text{weights}$ & \multicolumn{3}{c}{$w \sim \mathcal{U}\Bigl(-\sqrt{\frac{6}{n_{\text{in}}}},\,\sqrt{\frac{6}{n_{\text{in}}}}\Bigr)$, $n_{in}$ = number of input units}\\[1ex]
$\text{biases}$ & \multicolumn{3}{c}{$b \sim \mathcal{U}\Bigl(-\frac{1}{\sqrt{n_{\text{in}}}},\,\frac{1}{\sqrt{n_{\text{in}}}}\Bigr)$, $n_{in}$ = number of input units}\\
\bottomrule
\end{tabular}
\label{tab:hyperparam-exp}
\end{table} 
\begin{table}[h!]
\caption{Hyperparameter grids used for the baseline models in the UCI and Tabular data experiments}
\centering
\begin{tabular}{lc}
\toprule
Hyperparameters & UCI\\
\midrule
$\text{optimizer}$ & SGD \\
$\text{activation function}$ & GELU\\
\midrule
$z_{\text{\rm KARE}}$ & 0.1 \\
$\text{epochs}_{\text{\rm KARE}}$ & 300 \\
$\text{batch-size}_{KARE}$ & 32 if $n_{\text{total}} \geq 5000$ else $n_{\text{train}}$ \\
$\eta_{\text{\rm KARE}}$ & \{10, 5, 1, 0.1\} if depth $\geq$ 3, \{100, 10, 1, 0.1\} otherwise \\
\midrule
$\text{epochs}_{\text{MSE}}$ & 400 \\
$\eta_{\text{MSE}}$ & \{0.1, 0.01, 0.001, 0.0001\} \\
$\text{batch-size}_{MSE}$ & 32\\
\midrule
$\text{depth}$ & \{1,2,3,4,5\} \\
$\text{width}$ & \{32,64,128,256,512\} \\
\midrule
$\text{weights}$ & $w \sim \mathcal{U}\Bigl(-\sqrt{\frac{6}{n_{\text{in}}}},\,\sqrt{\frac{6}{n_{\text{in}}}}\Bigr)$, $n_{in}$ = number of input units \\
$\text{biases}$ & $b \sim \mathcal{U}\Bigl(-\frac{1}{\sqrt{n_{\text{in}}}},\,\frac{1}{\sqrt{n_{\text{in}}}}\Bigr)$, $n_{in}$ = number of input units \\
\bottomrule
\end{tabular}
\label{tab:hyperparam-exp-uci-openml}
\end{table}

Note that we use a smaller learning rate for deep networks to prevent chaotic dynamics in the gradients.

\subsection{Details}\label{det-sec}

\textbf{For the MNIST Dataset.} We sample $n$ images of digits 7 and 9 from the MNIST training dataset (image size $d = 24 \times 24$, edge pixels cropped, all pixels rescaled down to $[0, 1]$ and recentered around the mean value) and label each of them with $+1$ and $-1$ labels. We fit our models on $n = 1'000$ training data and compute the error on $1'000$ test observations. We repeat the experience $k=10$ times and report the average out-of-sample MSE for a grid of $\tilde \lambda$.

\vspace{1em}

\textbf{For the Higgs Dataset.} We randomly choose $n=1'000$ samples among those with no missing features marked with $-999$ from the Higgs training dataset. The samples have $d = 31$ features, and we normalize each feature column down to $[-1, 1]$ by dividing by the maximum absolute value observed among the selected samples. We replace the categorical labels ‘s’ and ‘b’ with regression values $+1$ and $-1$ respectively. We fit our models on $n = 1'000$ training data and compute the error on $1'000$ test observations. We repeat the experience $k=10$ times and report the average out-of-sample MSE for a grid of $\tilde \lambda$.

\textbf{For the UCI Datasets.}
We evaluate the out-of-sample performance of our models across a grid of hyperparameters, iterating over the 121 datasets. We adopt a single train-test procedure, using the same data splits (train, test) used by \cite{radhakrishnan2024mechanism} for hyperparameter selection. For binary classification tasks, we map the targets to ${-1,1}$ and apply mean squared error (MSE) to fit the DNN for the NTK model. For general classification tasks with NTK-KARE, we one-hot encode the classes using targets in ${-1,1}$. For explicit DNN training in general classification, we employ cross-entropy. For datasets with fewer than $100$ observations, we standardize the data using \textit{RobustScaler()} from sklearn. For datasets exceeding this threshold, we clip the data at the $5\%$ and $95\%$ percentiles and then standardize it to zero mean and unit variance. For NTK and NTK-KARE, we report the highest out-of-sample accuracy obtained over the grid of hyperparameters. During the model selection process, we exclude any hyperparameter combinations that cause training divergence, such as those resulting in singular matrix inversions.

\textbf{Experiment compute resources}
The simulations, MNIST, and Higgs experiments were conducted locally using an ASUS GeForce RTX 4080 Super TUF 16G GPU, though they can also be executed on CPU-only hardware. For the UCI datasets, we stratified the experiments based on dataset size. For datasets with fewer than $5'000$ observations, we trained the models using NVIDIA L40S GPUs on an internal compute cluster, encompassing 103 out of 121 datasets. Fitting both NTK-KARE and DNN/NTK models required a total of 48 GPU-hours.

For datasets with sample sizes between $5'000$ and $100'000$, we utilized NVIDIA H100 SXM5 GPUs on an internal compute cluster. Training all models on these datasets consumed a total of 199 GPU-hours. Lastly, for the \textit{miniboone} dataset with $130'000$ observations, we employed 2 NVIDIA H100 SXM5 GPUs per training task. Training NTK-KARE and DNN/NTK models on this dataset required 74 GPU-hours. In total, we consumed 321 GPU-hours to run our experiments.

\begin{table}[h!]
\caption{Datasets sources and download links}
\centering
\begin{tabular}{lcc}
\toprule
Datasets & Source & Link\\
\midrule
Higgs & Kaggle & \href{https://www.kaggle.com/competitions/higgs-boson/data}{Kaggle Higgs Boson Competition} \\
MNIST & PyTorch &\href{https://pytorch.org/vision/main/_modules/torchvision/datasets/mnist.html#MNIST}{See torchvision.datasets.mnist}\\
UCI & UC Irvine & \href{http://persoal.citius.usc.es/manuel.fernandez.delgado/papers/jmlr}{Delgado (2014) datasets} \\
\bottomrule
\end{tabular}
\label{tab:dataset-sources}
\end{table}


\newpage

\end{document}